\documentclass[submission,copyright,noncommercial]{eptcs}
\providecommand{\event}{ACT 2022} 
\usepackage{breakurl}             
\usepackage{underscore}           

\title{Regionalized Optimization}
\author{Grégoire Sergeant-Perthuis
\institute{Univ. Artois, UR 2462,\\ Laboratoire de Mathématiques de Lens (LML),\\ F-62300 Lens, France\\
}
\email{gregoireserper@gmail.com}
}

\def\titlerunning{Regionalized Optimization}
\def\authorrunning{G. Sergeant-Perthuis}
  \usepackage{mathtools}
  \usepackage{dsfont}
  \usepackage{amsmath}
  \usepackage{amsthm}
  \usepackage{amssymb}

  \theoremstyle{definition}
  \newtheorem{defn}{Definition}[section]
  
  \newtheorem*{defn*}{Definition}
 
  \theoremstyle{plain}
  \newtheorem{thm}{Theorem}[section]
  
  \newtheorem{prop}{Proposition}[section]
  \newtheorem*{prop*}{Proposition}

   \newtheorem*{cor*}{Corollary}
  \newtheorem*{theo*}{Theorem}
  \newtheorem*{thm*}{Theorem}
  
  \theoremstyle{remark}
  
  \newtheorem{rem}{Remark}[section]

  \newtheorem{nota}{Notation}[section]

 \usepackage[toc,page]{appendix}
\usepackage{calrsfs}
\usepackage{tikz-cd}

\newcommand{\p}{\mathbb{P}}

\newcommand{\R}{\mathbb{R}}
\newcommand{\A}{\mathcal{A}}

\usepackage{pstricks,pst-node,pst-tree}

\usepackage{bm}

\usepackage{relsize}

 \usepackage{xcolor}

 \newcommand\N{\mathbb{N}}
 \newcommand\cat[1]{\textbf{#1}}

\newcommand{\Pa}{\mathcal{P}}

\newcommand{\mes}{\cat{Mes}}
\newcommand{\Kern}{\cat{Kern}}
\newcommand{\Vect}{\cat{Vect}}

\newcommand{\dd}{\text{d}}
\newcommand{\im}{\operatorname{im}}

\newcommand{\colim}{\operatorname{colim}}

\makeatletter
\newtheorem*{rep@theorem}{\rep@title}
\newcommand{\newreptheorem}[2]{%
\newenvironment{rep#1}[1]{%
 \def\rep@title{#2 \ref{##1}}%
 \begin{rep@theorem}}%
 {\end{rep@theorem}}}
\makeatother

 \theoremstyle{plain}
\newreptheorem{thm}{Theorem}
\newreptheorem{prop}{Proposition}
\newreptheorem{cor}{Corollary}

\usepackage{comment}

\makeatletter
 \let\Ginclude@graphics\@org@Ginclude@graphics
\makeatother
\newcommand\imCMsym[4][\mathord]{%
  \DeclareFontFamily{U} {#2}{}
  \DeclareFontShape{U}{#2}{m}{n}{
    <-6> #25
    <6-7> #26
    <7-8> #27
    <8-9> #28
    <9-10> #29
    <10-12> #210
    <12-> #212}{}
  \DeclareSymbolFont{CM#2} {U} {#2}{m}{n}
  \DeclareMathSymbol{#4}{#1}{CM#2}{#3}
}
\newcommand\alsoimCMsym[4][\mathord]{\DeclareMathSymbol{#4}{#1}{CM#2}{#3}}

\imCMsym{cmmi}{124}{\CMjmath}
\imCMsym[\mathop]{cmsy}{113}{\CMamalg}
\imCMsym[\mathop]{cmex}{96}{\CMcoprod}
\alsoimCMsym[\mathop]{cmex}{97}{\CMbigcoprod}

\begin{document}

\maketitle

\begin{abstract}%
We propose a theoretical framework for non redundant reconstruction of a global loss from a collection of local ones under constraints given by a functor; we call this loss the regionalized loss in honor to Yedidia, Freeman, Weiss' celebrated article `Constructing free-energy approximations and generalized belief propagation algorithms' where a first example of regionalized loss, for entropy and the marginal functor, is built. We show how one can associate to these regionalized losses message passing algorithms for finding their critical points. It is a natural mathematical framework for optimization problems where there are multiple points of views on a dataset and replaces message passing algorithms as canonical ways of finding the optima of these problems. We explain how Generalized Belief propagation algorithms fall into the framework we propose and propose novel message passing algorithms for noisy channel networks.

\end{abstract}


\section{Introduction}

\subsection{Motivation}

Recent computational models of adaptive systems are based on the premise that these systems have an internal model of the state and dynamics of their environment that they infer through observations thanks to their sensors. They then use these beliefs in such a way to explore and exploit their environment based on preferences (active inference) \cite{FRISTON200670} \cite{ijms222111868}. It is common for these systems to have multiple sensors (photoreceptor, chemoreceptors...) and in particular these sensor are specifically sensitive to one `type' of information; multi-modal integration is the capacity of a system to synthesize information from diverse `sensory' information. It is believed that it allows for uncertainty reduction and improves robustness of the system's inference on the state of its environment. In particular, to account for multi-modal integration, computational models must take into account the fact that sensory evidence is a collection of `observations' of diverse sensors that sums up to a coherent view of their environment which can allow for non existing states or extended states. It has been previously remarked that sections of functors could be the good framework to describe data coming from multiple sensors that should be the result of different points of view of the same `thing' \cite{ROBINSON2017208}\cite{s20123418}. We propose to go one step further and ask how to solve control problems on the observations of each sensors taking into account that these problems must result in a control problem on the reconstructed coherent view of the environment. To do so we propose a general framework for reconstructing a `global' loss (on the coherent view of the environment), that we call Regionalized loss, that is the less redundant possible with respect to how `local' losses (on the observations of the sensors) `intersect' and exhibit a canonical message passing algorithms for finding the critical points of the Regionalized loss. More precisely, we consider data with a hierarchical structure given by a poset (a functor $F$ from the poset to the category of finite real vector spaces) for which one can define at each point of the hierarchy a loss function. We use the inclusion-exclusion formula for posets to build the Regionalized loss over the limit of the hierarchy ($\lim F$). We then characterize the critical points of the constrained optimization problem which consists in minimizing the Regionalized loss over $\lim F$. This characterization allows for a mapping of Lagrange multipliers to the space of constraints which we use to define the canonical message passing algorithm associate to the Regionalized loss.

\subsection{Related work}

The partition function of $N$ random variables $X_1\in E_1..X_N\in E_N$ with probability distribution $P\propto e^{-H}$ is $Z:=\sum_{x_1\in E_1..x_N\in E_N}e^{-H(x_1..x_N)}$; the number of terms in the sum increases exponentially with the number of random variables making its computation more and more complex. However one can compute approximations of the partition function using variational principles which relate the minimum of a variational free energy to $-\ln Z $; sometimes minimizing the free energy only gives approximations of the partition function. These methods fall in the field of (approximate) variational inference. The Belief propagation and Generalized Belief Propagation algorithms \cite{Yedidia} are (approximate) variational inference algorithms (Theorem 5 \cite{Yedidia}) that are related respectively to the Bethe Free energy and the region based approximation of free energy. The Bethe Free energy is a particular case of a region based approximation of free energy and so are the associated message passing algorithms. Interestingly they can also be seen as inference on data that is structured hierarchically: Bayesian networks in the first case, subsets of random variables in the second. In the simplest case (a graph without edges) it is simply Naive Bayes which is heavily used in active inference \cite{DaCosta2020}, it can also be used for multi-agent collaborations \cite{LEVCHUK201967}.

\subsubsection{Region based approximation of free energy.}

The Generalized Belief Propagation and its underlying variational free energy which is the region based approximation of free energy as introduced in \cite{Yedidia} is where our research for a unified framework for reconstructing global losses from local ones under functorial constraints started. The originality of the region based approximation of free energy resides on two points:

\begin{enumerate}
    \item firstly, the loss that approximates the relative entropy is a non-redundant
 global reconstruction of relative entropies of local probability distributions defined on elements of the hierarchy
    \item secondly, the loss is constrained on collections of probability distributions that are compatible under marginalization by restriction from collections of variables that rank higher in the hierarchy to subcollection of variables that rank lower.
\end{enumerate}

It is therefore natural for us to start by recalling constructions and results from \cite{Yedidia}, for which numerical interest is already established, so that that we can extend them to solve the much more general problem we pose.

The Generalized Belief Propagation is already an extension of the (loopy) Belief Propagation and region based approximations of free energy are extensions of the underlying Bethe free energies \cite{doi:10.1143/JPSJ.12.753} which are defined only on graphs of variables: particular cases of regions considered in the Generalized Belief Propagation. We will consider regions, $\A$, to be partially ordered sets (poset), i.e. sets with a binary relation denoted as $\leq$ that is
\begin{enumerate}
\item reflexive: for all $a\in \A$, $a\leq a$,
\item anti-symmetric: if $a\leq b$ and $b\leq a$ then $a=b$
\item transitive: if $c\leq b$ and $b\leq a$ then $c\leq b$
\end{enumerate}

In \cite{Yedidia} regions have a more restrictive definitions than the one we propose however they fall in our setting as they consider collections of subsets of all the variables. 

We decided to present only the Generalized Belief Propagation as there is too big of a gap between standard Belief Propagation on graphs and the result we will present in this article.

\begin{nota}
For any measurable space $E$ we will denote $\mathbb{P}(E)$ the measurable space of probability distributions defined on $E$. 
\end{nota}

When $E$ is a finite set, the entropy of a probability distribution, $p\in \p(E)$ over $E$ is defined to be, 

\begin{equation}
S(p)=-\sum_{x\in E} p(x)\ln p(x)
\end{equation}

For the rest of this section $E$ is a finite set; let $H:E\to \R$ be any random variable, where the notation $H$ stands for `Hamiltonian' in the reference to the statistical physics literature. Let $U\in \R$, let us recall that the MaxEnt principle \cite{Kesavan2009} is maximizing entropy over probability distributions for which the mean value of $H$ is fixed to be $U$. Restating it more formally, MaxtEnt is solving the following problem,

\begin{equation}
\sup_{\substack{p\in \mathbb{P}(E)\\ \mathbb{E}_p [H]=U}} -\sum_{x\in E} p(x)\ln p(x)
\end{equation}

The MaxEnt principle can be re-expressed, using Lagrange multipliers, as solving:

\begin{equation}\label{Maxent}
\inf_{p\in \p(E)} \mathbb{E}_{p}[\beta H] - S(p) 
\end{equation}

$\mathbb{E}_{p}[ H] - \frac{1}{\beta}S(p) $ is called the Gibbs free energy. The solution to this problem is the following celebrated expression for Gibbs measures

\begin{equation}
\forall x\in E \quad p^*(x) = \frac{e^{-\beta H(x)}}{\sum_{x\in E}e^{-\beta H(x)}}
\end{equation}

In the particular case of the MaxEnt principle presented above, the space on which the optimization of entropy is done is an affine subspace: it is defined by $\mathbb{E}_p[H]=U$ where $\mathbb{E}_p[H]$ is linear on $p$. More generally, one is interested in finding the optimum of entropy on subspaces of the probability space ($\mathbb{P}(E)$) that are meaningful for the kind of data one is interested under considerations \cite{Kesavan2009} \cite{DEMARTINO2018e00596}. Finding optimal probability distributions in a family of probability distributions that maximize entropy is known in Bayesian inference as variational inference. \\

Let us now consider a collection of random variables $(X_i\in E_i,i\in I)$ taking respectively values in finite sets $E_i$, indexed over a finite set $I$; the configuration space or universe in which the collection of random variables live is the product $E=\prod_{i\in I} E_i$ and we denote $p\in \mathbb{P}(E)$ the probability law of this collection. For any subset $a\subseteq I$ in the powerset of $I$, $\mathcal{P}(I)$, the restricted collection of random variables $X_a\in \prod_{i\in a} E_i$ has as probability distribution the marginal law of $X_a$ denoted $p_a$. We denote $\prod_{i\in a} E_i$ as $E_a$. The entropy can be rewritten using the inclusion-exclusion formula as,

\begin{equation}\label{rewriting-entropy-intro}
S(p)= \sum_{a\subseteq I}\sum_{b\subseteq a} (-1)^{\vert a\setminus b \vert}S_b(p_b)
\end{equation}

where $S_b(p_b)=-\sum_{x_b\in E_b} p_b(x_b)\ln p_b(x_b)$.

The previous equation expresses the `global' entropy $S(p)$ as an alternated sum of local ones $S(p_b),b\subseteq I$; the use of the inclusion-exclusion formula can be understood as a non redundant approximation of entropy built from the entropies of variables $X_b,b\subseteq I$,

\begin{equation}\label{approximation-powerset}
S_{GBP}(p_a,a\subseteq I):= \sum_{a\subseteq I}\sum_{b\subseteq a} (-1)^{\vert a\setminus b \vert}S_b(p_b)
\end{equation}

In this particular case the approximation is exact as $S(p)=S_{BP}(p_a,a\subseteq I)$. This is the particular case of the region based approximation of free energy following the definitions of \cite{Yedidia}, where the region is the powerset. In \cite{Yedidia} regions are collections $\A_I\subseteq \mathcal{P}(I)$ of subsets of $I$; any two subsets of a region can be ordered by the inclusion relation. The definition of the approximate entropy in Equation \ref{approximation-powerset} extends to other collections of subsets of $I$ as the inclusion-exclusion formulas exists more generally for any partially ordered sets; this construction is due to Rota in his celebrated article introducing partially ordered sets as foundations of combinatorics \cite{Rota}. He pointed out two linear operators defined on functions $f\in \bigoplus_{a\in \A}\R$ over a poset, $\A$: the sum over the poset, called the zeta function of the poset, for $a\in \A$,

\begin{equation}
\zeta(f)(a)=\underset{b\leq a}{\sum} f_b
\end{equation}

and its inverse (Proposition 2 \cite{Rota}), called the M\"obius inversion,

\begin{equation}
\mu(f)(a)=\underset{b\leq a}{\sum} \mu(a,b) f_b
\end{equation}

 The inclusion-exclusion formula of Equation  \ref{approximation-powerset} then extends to any region $\A_I\subseteq \mathcal{P}(I)$ (seen as a poset) as,

\begin{equation}\label{Approx-entropy}
S_{GBP}(p_a,a\in \A_I):= \sum_{a\in \A_I}\sum_{b\leq a} \mu(a,b) S_b(p_b)
\end{equation}
 
However this time, in general, $S_{GBP}(p_a,a\in \A_I)$ differs from $S(p)$.\\

The expression of $S_{GBP}$ makes sense for all $\prod_{a\in \A_I} \p(E_a)$, the product space of the probabilities over each subset in $\A_I$, however there are implicit relations between the marginal laws. Indeed for any $a,b\in \A_I$ such that $b\subseteq a$, marginalizing on $X_a$ and then on $X_b$ is the same than marginalizing directly on $X_b$. To make the previous statement precise we need to introduce some concepts and notations inspired by category theory (as we will detail in next section); the marginal on $X_a$ of $p$ is formally defined as the image measure of $p$ by the projection $\pi_a: E\to E_a$, denoted as ${\pi_{a}}_{*}:\p(E)\to \p(E_a)$. Similarly the projection $\pi^a_b: E_a \to E_b$, defined as $\pi^a_b(x_i,i\in a)= (x_i, i\in b)$, induces an application $\pi^a_b: \p(E_a)\to \p(E_b)$. The marginal of $X_a$ for $a\in \A_I$ is by definition ${\pi_{a}}_{*}(p)$ and by construction the following result holds: for any $b\in \A_I$ such that $b\subseteq a$,

\begin{equation}
{\pi^a_b}_{*}(p_a)= {\pi^a_b}_{*} {\pi_a}_{*}(p)= {\pi_b}_{*}(p)= p_b
\end{equation}

The underlying constrained subspace of $\prod_{a\in \A_I} \p(E_a)$, that we will denote, for now, as $C(\pi,\A_I)$ is explicitly defined as for any  $q=(q_a\in \p(E_a), a\in \A_I)$ ,

\begin{equation}\label{lim-marginal}
q\in C(\pi,\A_I)\iff  \forall a,b\in \A_I \text{ such that } b\subseteq a\quad {\pi^{a}_{b}}_{*}(p_a)= p_b
\end{equation}

\begin{defn}[Region based approximation of free energy]\label{region-based-optimization}
Let $I$ and $E=\prod_{i\in I} E_i$ be finite sets; for a region $\A_I\subseteq \mathcal{P}(I)$ and a collection of Hamiltonians $(H_a\in \R^{E_a} ,a\in \A_I)$, the region based approximation of free energy is the following loss, for $q=(q_a\in \p(E_a),a\in \A_I)$,

\begin{equation}
F_{\text{GBP}}(q):=\sum_{a\in \A} \sum_{b\leq a} \mu(a,b) \left(\mathbb{E}_{q_b}[H_b] - S_b(q_b)\right)
\end{equation}

The associated optimization problem is, 

\begin{equation}
\inf_{p\in C(\pi,\A_I)} F_{\text{GBP}}(p)
\end{equation}
\end{defn}



\subsubsection{Generalized Belief Propagation}

Generalized Belief Propagation are classes of algorithms that enable to find the critical points of the region based approximations of free energy (the optimization problem of Definition \ref{region-based-optimization}). A classical result states that fix points of this algorithm correspond to critical points of that free energy. Let us now recall the expression of this algorithm and the correspondence result.

As in the previous section $I$ is a finite collection of indices, $E$ a finite product of finite sets, $(H_a,a\in \A_I)$ is a collection of Hamiltonians and $\A_I\subseteq \mathcal{P}(I)$ is a collection of subsets of $I$. For two elements of $\A_I$, $a$ and $b$ such that $b\subseteq a$, two types of messages are considered: top-down messages $m_{a\to b}\in \R_{>0}^{E_b}$ and bottom-up messages $n_{b\to a}\in \R_{>0}^{E_a}$.  The update rule of the Generalized Belief Propagation is as follows, consider messages at times $t$, $(n_{b\to a}(t), m_{a\to b}(t) \vert b,a\in \A_I \text{ s. t. } b\subseteq a)$, they are related by the following relation,

\begin{equation}\label{GBP1}
\forall a,b\in \A_I, \text{s.t. }  b\subseteq a, \quad n_{b \to a}(t)= \prod_{\substack{c: b\subseteq  c\\ c\not \subseteq a}}m_{c\to b} (t)
\end{equation}

The multiplication of function $n_{b\to a}$ that have different domains is made possible because there is an embedding of $\R^{E_b}$ into $\R^{E_a}$ implicitly implied, as a set of cylindrical functions.  

One can define beliefs as,

\begin{equation}\label{GBP2}
\forall a \in \A_I, \quad b_a(t)= e^{-H_a}\prod_{\substack{b\in \A:\\ b\subseteq a}} n_{b\to a}(t)
\end{equation}

The beliefs are interpreted as probability distributions up to a multiplicative constant. The update rule is given by,

\begin{equation}\label{GBP3}
\forall a,b\in \A_I \text{ s.t. } b\subseteq a, \forall x_b\in E_b\quad m_{a\to b}(x_b,t+1)= m_{a\to b}(x_b,t)\frac{\sum_{y_a: \pi^a_b(y_a)=x_b} b_a(y_a,t)}{b_b(x_b,t)}
\end{equation}

The algorithm can be rewritten in a more condensed manner, updating only the top-down messages, for all $a,b\in \A_I$, such that $b\subseteq a$,

\begin{align}
 m_{a\to b}(x_b,t+1)&=  m_{a\to b}(x_b,t)\frac{\sum_{\substack{y_a: \pi^a_b(y_a)=x_b}} e^{-H_a(y_a)} \prod_{\substack{c\in \A:\\ c\subseteq a}}\prod_{\substack{d: c\subseteq  d\\ d\not \subseteq a}}m_{d\to c} (y_c,t)}{e^{-H_b(x_b)}\prod_{\substack{c\in \A:\\ c\subseteq b}}\prod_{\substack{d: c\subseteq  d\\ d\not \subseteq b}}m_{d\to c} (x_c,t)}
\end{align}

We will denote this update rule as $GBP$: $m(t+1)= GBP(m(t))$

\begin{thm}[Yedidia, Freeman, Weiss, Peltre]
Let $(m_{a\to b},a,b\in \A_I \text{s.t. } b\subseteq a)$ be a fix point of the Generalized Belief Propagation, i.e.

\begin{equation}
m= GBP(m)
\end{equation}

Let $(b_a,a\in \A_I)$ be the associated beliefs and let, for $a\in \A_I$, $p_a= b_a/ \sum_{x\in E_a} b_a(x)$ be the associated normalized beliefs. Then $(p_a,a\in \A_I)$ is a critical point of $F_{BP}$ under the constraint that $p\in C(\pi,\A_I)$.

\end{thm}

\begin{proof}
Theorem 5 \cite{Yedidia}, Theorem 5.15 \cite{Peltre}.

\end{proof}

\section{Proposed framework for optimization under functorial constraints: Regionalized Optimization}

The region based free energy approximations are particular to entropy and marginalization of probability distributions, which limits the possibilities with respect to where one can apply the previous constructions. However the important idea, or even principle, that can be retained from the the region based free energy approximations is to build a non redundant loss from local ones and replace global variables by local ones; it has a much broader scope of applications. However the region based free energy approximations fail to define a framework general enough so that these ideas can be applied to other losses over data that are not probability distributions. We propose to give a theoretical setting that accomplishes this motivation in a way general enough so that it can be applied to very different types of data and losses. We call it `Regionalized Optimization'. The contribution of this paper is to define this framework: the Regionalized loss, the associated messages passing algorithms and to prove that critical points of that loss correspond to fix points of these algorithms. In this document we keep the presentation the most general possible and do not want to explicitly restrict our construction to a certain type of dataset, we hope this work can be used as a blueprint for people who are willing to apply the previous principles to their data. Nevertheless in Appendix \ref{appendix:first} we give an example of how our constructions can lead to novel applications.\\

The constraint space of the previous section $C(\pi,\A_I)$ is a typical construction that appears in category theory: the limit of a certain functor. In what follows we will make use of the formalism of categories, functors and limits as it the necessary language to understand the objects we will be using and guides the intuition for the construction we propose, in particular for the message passing algorithms.

A poset is a particular case of a category where one defines an arrow, $b\to a$, between two elements $a,b\in \A$ when $b\leq a$. A functor $G$ from a poset $\A$ to the category of real vector spaces is a generalized function that sends elements of the poset $a\in \A$ to a vector space $G(a)$ and an arrow $b\to a$ to a linear application $G^b_a:G(b) \to G(a)$. The collection of vector spaces with arrows (morphisms) between two vector spaces that are linear transformations is called the category of ($\mathbb{R}-$) vector spaces and denoted as $\Vect$; if one only considers finite dimensional vector spaces, which will be our case, the associate (sub-)category is denoted $\Vect_f$. Given a functor $G$ from a poset $\A$ to $\Vect$, the collection of vectors $(v_a,a\in \A)$ that are compatible along the linear applications $G^b_a$ satisfy the following constraints: for all $a,b\in\A$ such that $b\leq a$, $G^b_a(v_b)= v_a$.

A cofunctor from a poset $\A$ to $\Vect$, is simply a functor from $\A^{op}$, the same poset with the opposite order, to the category of vector spaces. \\

In these terms, one remarks that $C(\pi, \A_I)$ is in fact $\lim \pi_*$. A dual construction is the colimit of a functor, denoted as $\colim$, which is a limit when the arrows of the target category are opposed; we will not go into the details of its definition but one can refer to \cite{MacL} for more on the subject. A simple example of a limit is the product of two sets $A\times B$ and of a colimit is the disjoint union $A\CMcoprod B$.

\subsection{Overview of results}

\subsubsection{Regionalized loss}

In this paper we propose,
\begin{enumerate}
    \item to build a global cost function, called Regionalized loss, from local cost functions over a set of local variables that are compatible in some way
    \item to define message passing algorithms which fix points are the critical points of the Regionalized loss, so that for finding the critical points of this loss one can iterate the message passing algorithms. 

\end{enumerate}

More precisely, let $F$ be a cofunctor from a finite poset $\A$ to $\Vect_f$, the category of finite ($\mathbb{R}-$) vector spaces, and let $(f_a:F(a)\to \R,a\in \A)$ be a collection of local cost functions, the Regionalized loss (with respect to $F$ and $(f_a,a\in \A)$) is:

\begin{equation}\label{chapitre-8:intro:local-optimization}
\begin{array}{ccccc}
f_R& : &\prod_{a\in \A} F(a) & \to & \R \\
& & (x_a,a\in \A) & \mapsto &\sum_{a\in \A} \sum_{b\leq a} \mu(a,b) f_b(x_b)\\
\end{array}
\end{equation}

We will motivate our choice in Section \ref{chapitre-8:section-1-1}. The associated `Regionalized' Optimization problem is,

\begin{equation}
\sup_{x\in \lim F} f_R(x)
\end{equation}

One can think of $f$ as the less redundant reconstruction of a global cost function from the local cost function $(f_a,a\in \A)$.\\

\subsubsection{Characterisation of critical points}

Let $G$ be a functor from $\A$ to $\Vect$; the dual functor, $G^*$, sends any element $a\in \A$, the set of linear applications (forms) from $G(a)$ to $\R$, denoted as $G(a)^*$,

\begin{equation}
G^*(a):= G(a)^*
\end{equation}

and to any two elements $a,b\in \A$ such that $b\leq a$, ${G^*}^{b}_{a}$, the dual map of $G^b_a$ usually denoted as ${G^b_a}^*$, which sends a linear form $l\in G(a)^*$ to $l\circ G^b_a\in G(b)^*$; to sum it up,

\begin{equation}
 {G^*}^{b}_{a}=  {G^b_a}^*
\end{equation}

In the following theorem, we show that, when $(f_a:F(a)\to \R,a\in \A)$ is a collection of differentiable applications, the critical points of the Regionalized loss are characterized `locally' up to `messages' on the (opposed) arrows of the poset.

\begin{repthm}{thm-local-optimization}
 
An element $x\in \lim F$ is a critical point of the Regionalized loss $f_R$ with respect to a cofunctor $F$ and a collection of differentiable applications $(f_a:F(a)\to \R,a\in \A)$  if and only if there is $(l_{a\to b}\in \bigoplus_{\substack{a,b:\\ b\leq a}} F(b)^{*})$ such that for any $a\in \A$,

\begin{equation}
d_xf_a =\sum_{b:b\leq a}{F^a_b}^{*}\left( \sum_{c:c\leq b} {F^b_c}^{*} l_{b\to c}- \sum_{c:c\geq b} l_{c\to b} \right)
\end{equation}

 \end{repthm}

\subsubsection{Message passing algorithms}

We will now explain how the message passing algorithms we propose for finding the critical points of the Regionalized losses are defined. To do so we assume that the local cost functions $f=(f_a,a\in \A)$ are such that there is a collection of functions $g=(g_a,a\in \A)$ that inverse the relation induced by differentiating these function, i.e. that for any $a\in \A$,

\begin{equation}\
\forall x,y\quad d_xf_a = y \iff x= g_a(y)
\end{equation}

We consider the Regionalized loss, $f_R$, with respect to $(f_a,a\in \A)$ and $F$; we define the following message passing algorithm; there are two messages,

\begin{enumerate}
\item $m(t)\in \bigoplus_{\substack{a,b:\\ b\leq a}} F(b)^{*}$
\item $n_t\in \bigoplus_{a\in \A} F(a)^*$
\end{enumerate}

For any $a,b\in \A$ such that $b\leq a$, we define the following update rule: 

\begin{align}
m_{a\to b}(t+1) &= m_{a\to b} (t) +F^a_bg_a(n_t(a))-g_b(n_t(b)) \\
n_t(a)&= \sum_{b:b\leq a} \sum_{c:c\leq b} {F^a_c}^{*}m_{b\rightarrow c}(t) -\sum_{b:b\leq a}\sum_{c:c\geq b} {F^a_b}^{*} m_{c\rightarrow b}(t)
\end{align}

We show that the fix points of this algorithm are critical points of $f_R$ over $\lim F$. \\

Let us now go into the details of the claims of this Section.


\subsection{Main contribution}\label{chapitre-8:section-1-1}

\subsubsection{Definition of M\"obius inversion for functors and cofunctors}
\begin{defn}[M\"obius inversion associated to a functor]
Let $\A$ be a finite poset, let $\mu$ be its M\"obius inversion. Let $G: \A\to \Vect$ be a functor from a finite poset to the category of ($\R$-) vector spaces; we introduce $\mu_{G}: \bigoplus_{a\in \A} G(a)\to \bigoplus_{a\in \A} G(a)$ to be such that for any $a\in \A$ and $v\in \bigoplus_{a\in \A} G(a)$,

\begin{equation}
\mu_{G}(v)(a):=\sum_{b\leq a} \mu(a,b) G^b_a(v_b)
\end{equation}

\end{defn}


\begin{prop}
Let $G: \A\to \Vect$ be a functor from a finite poset to the category of vector spaces, $\mu_{G}$ is invertible and its inverse, denoted $\zeta_{G}$, is defined as follows: for any $a\in \A$ and $v\in \bigoplus_{a\in \A} G(a)$,

\begin{equation}
\zeta_{G}(v)(a)=\sum_{b\leq a} G^b_a(v_b)
\end{equation}
\end{prop}

\begin{proof}
Let $v\in \bigoplus_{a\in \A} G(a)$ and $a\in \A$,

\begin{equation}
\zeta_{G}\mu_G(v)(a)= \sum_{b:b\leq a}\sum_{c:c\leq b} \mu(b,c) G^b_aG^c_b(v_c)
\end{equation}

therefore,

\begin{equation}
\zeta_{G} \mu_{G}(v)(a)= \sum_{c\leq a}\left(\sum_{b: \ c\leq b\leq a} \mu(b,c)\right) G^c_a(v_c)= G^a_a(v_a)
\end{equation}

Furthermore,

\begin{equation}
\mu_{G}\zeta_{G}(v)(a)= \sum_{b\leq a} \mu(a,b)\sum_{c\leq b} G^c_a(v_c)= v_a
\end{equation}

Which ends the proof.

\end{proof}

\begin{rem}\label{chapitre-8:change-order-poset}
For any poset $(\A,\leq)$ one can oppose the order relation: for any $a,b\in \A$, 

\begin{equation}
a \leq_{op} b\iff b\leq a
\end{equation}

 $\leq_{op}$ is also denoted as $\geq$ and the corresponding poset, $(\A,\geq)$, as $\A^{op}$. One shows that, for any $a,b\in \A$ such that $a\geq b$,

\begin{align}
\zeta_{\A^{op}}(b,a)= \zeta_{\A}(a,b)\\
\mu_{\A^{op}}(b,a)= \mu_{\A}(a,b)
\end{align}

In particular for any functor $G: \A\to \Vect$ from a finite poset to the category of vector spaces,

\begin{equation}
\mu_{G^{*}}= (\mu_G)^{*}
\end{equation}

as for any $(l_a\in G(a)^{*},a\in \A)$,
\begin{equation}
\sum_{a\in \A} \sum_{b\leq a} \mu(a,b) l_aG^b_a= \sum_{b\in \A} \sum_{a\geq b} \mu(a,b) {G^{*}}^a_b(l_a)
\end{equation}
\end{rem}

\subsubsection{Regionalized loss}\label{Local-optimization-cost-function-subsection}

\begin{defn}[Definition of the Regionalized loss]\label{Local-optimization-cost-function}

Let $F:\A\to \Vect_f$ be a cofunctor over a finite poset, $\A$, to $(\R-)\Vect_f$, the category of finite dimensional ($\R$-)vector spaces, let $(f_a: F(a)\to \R, a\in \A)$ be a collection of functions that we will call local cost functions. We define the Regionalized Optimization problem with respect to $F$ and $(f_a: F(a)\to \R), a\in \A)$, to be the following optimization problem,

\begin{equation}\label{Local-optimization-cost-function-equation}
\sup_{x\in \lim F} \sum_{a\in \A} c(a) f_a(x_a)
\end{equation}

where for any $a\in \A$, 

\begin{equation}
c(a)= \sum_{b\geq a} \mu(b,a) 
\end{equation}

$f_R=\sum_{a\in \A} c(a) f_a$ is the Regionalized loss. 

\end{defn}


\begin{rem}
In the previous definition, we considered cofuntors over a given poset because in the applications it is this way that it appears; however, as explained in Remark \ref{chapitre-8:change-order-poset} we could also state the previous definition for functors $F:\A^{op}\to \Vect_f$ and in this case, 
\begin{equation}
c(a)= \sum_{b\leq_{op} a} \mu_{op}(a,b)
\end{equation}

\end{rem}

\begin{rem}
The constraint $x\in \lim F$ makes the Regionalized Optimization problem a constrained one. It is the simplest problem one can consider when wanting to follow the principles stated before (less redundant reconstruction of local problems); one could want to add more constraints to the Regionalized Optimization depending on the specificities of the problem at hand (for example linear ones) however some constraints are `more compatible' with $F$ than others; it is something that we will evoke, without going into details, when considering the application to noisy channel networks in Appendix \ref{appendix:first}.
\end{rem}

\subsubsection{Critical points of the Regionalized Optimization problem}\label{chapitre-8:local-optimization}

The following Theorem is the central result that motivates the choice of message passing algorithms we define for finding critical points of Regionalized Optimization problems.

\begin{thm}\label{thm-local-optimization}
An element $x\in \lim F$ is a critical point of the Regionalized Optimization problem with respect to a cofuntor $F$ from a finite poset $\A$ to $\Vect_f$ and a collection $f=(f_a,a\in \A)$ if and only if,

\begin{equation}\label{thm-local-optimization-critical-points}
\big[\mu_{F^*}d_xf\big]|_{\lim F}=0
\end{equation}

In the previous equation for any $a\in \A$ and $v\in \bigoplus_{a\in \A} F(a)$, $d_xf(v)= \sum_{a\in \A} d_x f_a(v_a)$ and $d_xf$ is implicitly identified to $\bigoplus_{a\in \A} d_{x_a}f_a$ as there is a natural identification $\big(\bigoplus_{a\in \A} F(a)\big)^* \cong \bigoplus_{a\in \A} F(a)^*$.
\end{thm}

\begin{proof}

See Appendix \ref{appendix:second}.\\

\end{proof}

We will now show how we can parameterize linear forms that cancel over $\lim F$ so that we can reexpress Equation \ref{thm-local-optimization-critical-points} in a way that will allow for defining message passing algorithms. The following exact sequence holds,

\begin{equation}\label{exact-sequence}
0\rightarrow \lim F\rightarrow \bigoplus_{a\in \A}F(a)\overset{\delta_F}{\rightarrow} \bigoplus_{\substack{a,b\in \A\\ a\geq b}} F(b)
\end{equation}

where for any $v\in  \underset{\substack{a,b\in \A\\ a\geq b}}{\bigoplus} F(b)$ and $a,b\in \A$ such that $b\leq a$, 

\begin{equation}
\delta_F(v)(a,b)= F^a_b(v_a)- v_b
\end{equation}

Let us drop the $F$ in $\delta_F$ and write $\delta$ instead.
Saying that sequence in Equation $\ref{exact-sequence}$ is an exact sequence is the same than saying that,

\begin{equation}
\ker \delta= \lim F
\end{equation}

In orther words, $\delta$ gives an implicit definition of $\lim F$ as it is the set where $\delta$ vanish. In $\Vect$ it is a fact that the dual of the previous exact sequence is also exact:

\begin{equation}
0\leftarrow (\lim F)^{*} \leftarrow \bigoplus_{a\in \A}F(a)^{*}\overset{\dd_F}{\leftarrow} \bigoplus_{\substack{a,b\in \A\\ a\geq b}} F(b)^{*}
\end{equation}

with $\dd =\delta ^{*}$. This means that, 

\begin{equation}
\bigoplus_{a\in \A}F(a)^{*}/\im \dd = (\lim F)^*
\end{equation}

Therefore $\lim F\cong \colim F^\star$. The explicit definition of $\dd$ is, for any $l_{a\rightarrow b}\in \underset{{\substack{a,b\in \A\\ a\geq b}}}{\bigoplus} F(b)^{*}$ and $a\in \A$,

\begin{equation}
\dd l(a)= \sum_{a\geq b} {F^a_b}^{*}(l_{a\rightarrow b}) -\sum_{b\geq a} l_{b\rightarrow a}
\end{equation} 

So in particular on can restate Equation \ref{thm-local-optimization-critical-points} as,

\begin{equation}
\mu_F^{*}d_xf\in \im \dd
\end{equation}

which can be rewritten as the fact that there is $(l_{a\to b}\in F(b)^{*} \vert a,b\in \A \text{ s. t. } b\leq a)$ such that, 

\begin{equation}\label{expression-standard-critical-points}
d_xf= \zeta_{F^{*}} \dd l
\end{equation}

\subsubsection{Message passing algorithms}\label{message-passing-regionalized}

We will now show how we can derive a collection of message passing algorithms from the expression of the critical points Equation \ref{expression-standard-critical-points} when one can inverse the following relations, 

\begin{equation}
\forall a\in \A \forall x_a\in F(a), y_a\in F(a)^* \quad d_{x_a}f_a = y_a
\end{equation}

as it is the case in the examples we crossed paths with (see Appendix \ref{appendix:first} for one example). Therefore we will assume that there is $g$ such that for all $x\in \bigoplus_{a\in \A} F(a)$ and $y\in \bigoplus_{a\in \A} F(a)^*$,



\begin{equation}\label{critical-points-with-mesasges}
d_xf = y \iff x= g(y)
\end{equation}

In particular one can show that this implies that for any $a\in \A$ and $ x\in \bigoplus_{a\in A}F(a)$, 

\begin{equation}
 g_ad_{x_a}f_a= x_a
\end{equation}

The relation that relates Lagrange multipliers $l$ and the constraint is then given by composing by the functions that defines the constraints, which is $\delta_F$. The relation is explicitly given by $\delta_F g\zeta_{F^*}\dd_F$  and it sends a Lagrange multiplier $l\in \bigoplus_{\substack{a,b\in \A\\ a\geq b}}F(b)^*$ to a constraint $c\in  \bigoplus_{\substack{a,b\in \A\\ a\geq b}}F(b)$ defined as, for $a,b\in \A$ such that $b\leq a$,
\begin{equation}
c(a,b)= \delta_F g\zeta_{F^*}\dd_F l (a,b)=F^a_bg_a(\zeta_{F^*}\dd_Fl(a))-g_b(\zeta_{F^*}\dd_F l(b)))
\end{equation}


The particular constraint we are interested in is $c=0$, i.e. 

\begin{equation}\label{relation-constraint-messages}
\delta_F g\zeta_{F^*}\dd_F l=0
\end{equation}

as in implies that $g\zeta_{F^*}\dd_F l\in \lim F$: if we can find a Lagrange multiplier $l\in \bigoplus_{a\in \A}F(a)^* $ such that $\delta_F g\zeta_{F^*}\dd_F l=0$ we can find a critical point of the global optimization problem.\\

We will now show that any algorithm, $dl_{t+1}=h(dl_t)$, such that Equation \ref{relation-constraint-messages} is satisfied for its fix points defines a critical point.

\begin{thm}\label{message-passage-regionalized-optimization}
Let us consider $F$ a cofuntor from a finite poset $\A$ to $\Vect_f$ and a collection of local cost functions $(f_a:F(a)\to  \R,a\in \A)$. Let $h: \bigoplus_{a\in \A}F(a)^*\to \bigoplus_{a\in \A} F(a)^*$ be a function; one can build an algorithm from $h$ by considering the iterated sequence 

\begin{equation}
\forall t\in \N \quad l_{t+1}= h(l_{t})
\end{equation}

with $l_0\in \bigoplus_{a\in \A} F(a)^*$ being an initial condition.\\

Assume that any fix point of $h$, $l^*\in \bigoplus F(a)^*$ (which is a fix point of the algorithm), is such that

\begin{equation}
\delta_F g\zeta_{F^*}\dd_F l^*=0
\end{equation}

 Then, 

\begin{equation}
x_*=g(\zeta_{F^*}\dd_F l^*)
\end{equation}

is a critical point of the associated Regionalized Optimization problem associated to $F$ and $(f_a,a\in \A)$.

Let us recall that a fix point, $l^*$, of $h$ is by definition an element such that,

\begin{equation}
l^*=h(l^*)
\end{equation}

\end{thm}

\begin{proof}

Let $x^*=g(\zeta_{F^*}\dd l^*)$ then,

\begin{equation}
d_{x^*}f=\zeta_{F^*}\dd_F l^*
\end{equation}

and for any $a,b\in \A$ such that $b\leq a$,

\begin{equation}
F^a_b(x^*_a)= x^*_b
\end{equation}

Therefore $x^*$ is a critical point of the Regionalized Optimization problem associated to $F$ and $(f_a,a\in \A)$. Which ends the proof.\\
\end{proof}

In Theorem \ref{message-passage-regionalized-optimization}, any algorithm that enables to find a set of Lagrange multiplier such that the related constraints, $c=\delta_F g\zeta_{F^*}\dd_F l$, are the constraints of the problem we are interested in, here $c=0$, is a good choice of algorithm. For example one could consider a Newton method for finding the roots of $\delta_F g\zeta_{F^*}\dd_F$, this algorithm satisfies the conditions of Theorem \ref{message-passage-regionalized-optimization}. 

Other than a Newton method one can consider the following `generic' class of message passing algorithms: for any $F$ and collection of functions $(f_a,a\in \A)$, let us define the following algorithm,

\begin{equation}
l(t+1)- l(t):= \delta_F g\zeta_{F^*}\dd_F l(t)
\end{equation}

Fix points of this algorithm, $l^*$, satisfy $\delta_F g\zeta_{F^*}\dd_F l^*=0$. We also write $l(t)$ as $l_t$. More explicitly this algorithm, associated to $F$ and $(f_a,a\in \A)$, is defined by, for any $a,b\in \A$ such that $b\leq a$,

\begin{equation}
l_{a\to b}(t+1)= l_{a\to b} (t) +F^a_bg_a(\zeta_{F^*}\dd_F l_t (a))-g_b(\zeta_{F^*}\dd_F l_t(b))) 
\end{equation}

\begin{rem}
In particular for the region based free energy approximation the associated `generic' class of message passing algorithms is the Generalized Belief Propagation. One can convince oneself of this result by using sums instead of products:

\begin{equation}
 \ln b_a= H_a+\sum_{b\leq a} \sum_{\substack{c: b\leq c\\ c\nleq a}} \ln m_{c\to b}
 \end{equation}
 
and remarking that,
 
\begin{equation}
\zeta \dd \ln m (a)= \sum_{b\leq a} \sum_{\substack{c: b\leq c\\ c\nleq a}} \ln m_{c\to b}
\end{equation}

The previous equation is similar to the Gauss formula exhibited in Peltre's PhD thesis Proposition 2.3 (Gauss Formula) \cite{Peltre}.\\
\end{rem}

\section{Conclusion}

In this article we answered the question: how to build the less redundant optimization problem over a global reconstruction of data from local/partial data, given a collection of losses defined on the partial data. We showed how this problem is specially relevant for multi-sensor information integration. We gave canonical message passing algorithms to solve this optimization problem and gave an example of application of the framework we propose.

We are working on implementing this framework to several applications; an other application, that we did not present in this paper, is a PCA adapted to filtered data, for example time series where data is expected to in an ordered fashion: first $X_1$ then $X_2$ then $X_3$... An other direction we are considering is to extend the previous results by replacing posets with categories in general; results from \cite{Leinster} where an inclusion-exclusion formula is defined for some categories could be very helpful. One other extension, that would have many applications for adaptive systems, would be not to assume that the functor is given a priori but that it is also a variable on which to optimize.

\section{Acknowledgments}
The author would like to thank Daniel Bennequin, Juan Pablo Vigneaux for our numerous discussions and most particularly Olivier Peltre. He would like to thank warmly an anonymous reviewer for suggesting how to improve the presentation of the results of this paper. 


\bibliographystyle{eptcs}
\bibliography{bibliography}

\newpage

\appendix

\section{New algorithm: Regionalized Optimisation for noisy channel networks}\label{appendix:first}

\subsection{Region based free energy approximation as a global optimization problem}

The region based free energy approximation minimization is the Regionalized Optimisation version of the maximum of entropy under energy constraint for the presheaf given by marginalizations. Let $I$ be a finite set and let $E=\prod_{i\in I} E_i$ be a product of finite sets. Let us denote $\mes$ the category that has as object measurable spaces and as morphisms measurable applications.  Let $\A\subseteq \Pa(I)$ be a subposet of the powerset of $I$ and for any $a\in \A$ let $F(a)= \p_{>0}(E_a)$ the set of strictly positive probability densities over the finite set $E_a$. For any $a,b\in \A$ such that $b\subseteq a$, we shall note the marginalization, ${\pi^a_b}_{*}$ restricted to $\p_{>0}(E_a)$, that sends $\p_{>0}(E_a) \to \p_{>0}(E_b)$, as $F^a_b$. $F$ is a cofunctor from $\A$ to $\mes$. In order to apply the results of this article, we extend it to a cofuntor $\tilde{F}$ from $\A$ to $\Vect$ by defining for any $a\in \A$, $\tilde{F}(a)= \R^{E_a}$ and for $b\in \A$ such that $b\leq a$, any $q\in F(a)$ and any $x_b\in E_b$,

\begin{equation}
\tilde{F}^a_bq(x_b)= \sum_{\substack{y\in E_{a}:\\ y_b=x_b}} q(y)
\end{equation}

Let us note $1_a\in \R^{E_a}$ the constant function that equals to $1$ and let us denote $\langle 1_a\vert$ the associated linear form for the canonical scalar product on $\R^{E_a}$. The fact that $F$ is a subobject of $\tilde{F}$ can be restated as saying that for any $a,b\in \A$ such that $b\leq a$, $\langle 1_b\vert \tilde{F}^a_b =  \langle 1_a\vert$. In other words the constraint that $p\in F(a)$ is compatible with $\tilde{F}$.

Let $(H_a\in \R^{E_a},a\in \A)$ be a collection of Hamiltonians; the region based free energy approximation corresponds to the Regionalized loss with respect to the cofunctor $F$ and local losses $(f_a:q_a\to \sum q_a H_a - S_a(q_a),a\in \A)$, defined as for any $q\in \bigoplus_{a\in \A} \R_{>0}^{E_a}$,

\begin{equation}\label{Bethe-free-energy}
F_{GBP}(q)=f_R(q)= \sum_{a\in \A}c(a) \left(\sum q_a H_a - S_a(q_a)\right)
\end{equation}

The associated Regionalized optimization problem is, 

\begin{equation}\label{Bethe-free-energy-optimization}
\inf_{p\in \lim F} f_R(p)
\end{equation}

which can be rewritten as,

\begin{equation}\label{appendix:regionalized-expression}
\inf_{\substack{q\in \lim \tilde{F}\\ \forall a\in \A q_a> 0 \\ \underset{x_a\in E_a}{\sum} q_a(x_a)=1}} \sum_{a\in \A}c(a) \left(\sum_{x_a\in E_a}q_a(x_a) H_a(x_a) + \sum_{x_a\in E_a} q_a(x_a) \ln q_a(x_a) \right)
\end{equation}

\subsection{Global Optimisation and Belief propagation for noisy channel networks}\label{Bethe-Kern}

We denote $\Kern$ the category of probability kernels that has as object measurable spaces and as morphisms between two measurable spaces $E$ and $E_1$ probability kernels, i.e measurable application from $E$ to $\mathbb{P}(E_1)$. We denote a probability kernel $k: E\to E_1$, between two finite sets, as $(k(\omega_1\vert\omega), \omega_1\in E_1, \omega\in E_1)$.
The Regionalized optimization (Definition \ref{Local-optimization-cost-function}) of the entropy with energy constraints is in fact more general than the region based approximation of free energy: simply changing the category in which the functor $F$ takes values from $\mes$ to $\Kern$ allows for a new class of algorithms where marginalization are replaced by noisy channels. This might seem a rewriting trick but it is not; it is true and evident only now that we have developed the whole theoretical framework. We believe it is a notable improvement with respect what exists in the literature as we have not found any equivalent so far. \\
What changes with respect to the previous subsection is the choice of the cofunctor; let us now clarify how it is built. Let us consider cofunctors $F$ from $\A$ to $\Kern$ that take values in finite sets, i.e. for any $a\in \A$, $F(a)$ is a finite set and such that for any $a,b\in \A$ such that $b\leq a$ and any $\omega\in F(a), \omega_1\in F(b)$,

\begin{eqnarray}
F^a_b(\omega_1\vert \omega) >0  \nonumber\\
\sum_{\omega_2\in F(b)} F^a_b(\omega_2\vert \omega)=1
\end{eqnarray}


Any cofunctor, $F$, from a poset $\A$ to the category of probability kernels is a noisy channel network, where the channels are the probability kernels from $F(a)$ to $F(b)$ for every couple of elements of the network $a,b$ that are related ($b\leq a$).

Let us define for $a,b\in \A$ such that $b\leq a$, ${F_*}^a_b: \p_{>0}\left(F(a)\right)  \to  \p_{>0}\left(F(b)\right)$ defined for any $p\in \mathbb{P}\left(F(a)\right)$ as,
\begin{equation}
{F_*}^a_b(p):= F^a_b \circ p = \sum_{\omega \in F(a)} F^a_b(.\vert \omega)p(\omega)
\end{equation}

${F_*}^a_b$ can be extended to a cofuntor from $\A$ to $\Vect_f$ as follows,

\begin{equation}
\begin{array}{ccccc}
\tilde{F}^a_b& : &\R^{F(a)} & \to & \R^{F(b)}\\
& & f & \mapsto &\left(\sum_{\omega\in F(a)}f(\omega)F^a_b(\omega_1\vert \omega), \omega_1\in F(b)\right)\\
\end{array}
\end{equation}

For any $a,b\in \A$ such that $a\leq b$, $\langle 1_b\vert \tilde{F}^a_b =  \langle 1_a\vert$ as $F^a_b$ is a probability kernel, therefore $F_*$ is a subobject of $\tilde{F}$. In other words the constraint that $p\in \p_{>0}(F(a))$ is compatible with $\tilde{F}$.\\

\begin{rem}
 Let us remark that the cost functions are not defined on all $\R^{F(a)}$ but only on $\R_{>0}^{F(a)}$, however $\prod_{a\in \A} \R_{>0}^{F(a)}\cap \lim \tilde{F}$ is an open subset of $\tilde{F}$ therefore the characterization of the critical points stated in Section \ref{chapitre-8:local-optimization}, Theorem \ref{thm-local-optimization}, Theorem \ref{message-passage-regionalized-optimization} still holds.\\
\end{rem} 

The Regionalized Optimization associated to $(f_a:q_a\to \sum q_a H_a - S_a(q_a),a\in \A)$ and $\tilde{F}$ is therefore well defined and has the same expression as Equation \ref{appendix:regionalized-expression}.
 
To conclude let us explicit the message passing algorithm associated to the Regionalized Optimisation of free energy on noisy channel networks. To do so let us remark that the dual of $F_*$, is in fact the functor of conditional expectations: let $a,b\in \A$ be such that $b\leq a$, for any $f\in \R^{F(b)}$ and $\omega\in F(a)$

\begin{equation}
{F^a_b}^*(f)(\omega)=  \sum_{\omega_1\in F(b)} f(\omega_1) F^a_b(\omega_1\vert \omega)
\end{equation}

The canonical message passing algorithm associated to this Regionalized Optimization problem can be reexpressed in a way that is more familiar to people are ued with the General Belief Propagation (GBP). For two elements of $\A$, $a$ and $b$ such that $b\leq a$, we consider two types of messages: top-down messages $m_{a\to b}\in \R_{>0}^{F(b)}$ and bottom-up messages $n_{b\to a}\in \R_{>0}^{F(a)}$. The update rules are similar to the ones for the GBP algorithm (Equations \ref{GBP1},\ref{GBP2}); consider messages at times $t$, $(n_{b\to a}(t), m_{a\to b}(t) \vert b,a\in \A \text{ s. t. } b\leq a)$, they are related by the following relation,

\begin{align}
\forall a,b\in \A, \text{s.t. }  b\leq a, \quad n_{b \to a}(t)&= \prod_{\substack{c: b\leq c\\ c\nleq a}}{F^a_b}^* m_{c\to b} (t)\\
\forall \omega_1\in F(a)\quad n_{b\to a}(\omega_1)=& \prod_{\substack{c: b\leq c\\ c\nleq a}} \sum_{\omega\in F(b)} m_{c\to b}(\omega) F^a_b(\omega\vert \omega_1)
\end{align}

One then defines beliefs as ,

\begin{equation}
\forall a \in \A, \quad b_a(t)= e^{-H_a}\prod_{\substack{b\in \A:\\ b\leq a}} n_{b\to a}(t)
\end{equation}

 The update rule is,

\begin{align}
\forall a,b\in \A, \text{s.t.} b\leq a\quad m_{a\to b}(t+1)&= m_{a\to b}(t)\frac{F^a_b b_a(t)}{b_b(t)}\\
\forall \omega \in F(b) \quad m_{a\to b}(\omega,t+1)&= m_{a\to b}(\omega,t)\frac{\sum_{\omega_1\in F(a)}  b_a(\omega_1,t)F^a_b(\omega\vert \omega_1)}{b_b(\omega,t)}
\end{align}

A consequence of Section \ref{message-passing-regionalized} is that the fix points of this message passing algorithm are critical points of the Regionalized (free energy) loss. Once again it is important to insist that these Regionalized free energies and their associated message passing algorithms have a much greater reach for applications than the region based approximation of free energy and Generalized Belief Propagation as:
\begin{enumerate}
    \item the $F(a)$ can be any measurable spaces, it is not necessarily a product of spaces with indices in a subsets $a\subseteq I$ for a collection of subsets of $I$ ($F(a)\neq \prod_{i\in a} E_i$); the applications $F^a_b$ can be any measurable map, not just projections,
\item the morphisms $F^a_b: F(a)\to F(b)$ can be more general than measurable functions, as they can be probability kernels that allow uncertainty on the output of an element of $F(a)$.
\end{enumerate}

\section{Proof of Theorem 2.1}\label{appendix:second}

Let us note $c(u)=\sum_{a\in \A} c(a) f_a(u_a)$ for $u\in \lim F$, then any critical point $x\in \lim F$ is defined by,

\begin{equation}
d_xc|_{\lim F}=0
\end{equation}

Then for any $u\in \lim F$, 

\begin{equation}
\sum_{a\in \A}\sum_{b\leq a} \mu(a,b) d_xf_b(u_b)=0
\end{equation}

Furthermore for any $b\leq a$, $F^a_b(u_a)= u_b$ therefore,

\begin{equation}
\sum_{a\in \A}\sum_{b\leq a} \mu(a,b) d_xf_bF^a_b(u_a)=0
\end{equation}

in other words,

\begin{equation}
\sum_{a\in \A}\mu_F^{*}(d_xf)(u_a)=\sum_{a\in \A} \mu_{F^*}(d_xf)(u_a)=0
\end{equation}

which can be restated as,
\begin{equation}
\mu_{F^*}d_xf|_{\lim F}=0
\end{equation}

Which ends the proof.\\

\end{document}